\documentclass[letterpaper]{article} 
\usepackage{aaai25}  
\usepackage{times}  
\usepackage{helvet}  
\usepackage{courier}  
\usepackage[hyphens]{url}  
\usepackage{graphicx} 
\usepackage{natbib}  
\usepackage{caption} 
\usepackage{algorithm}

\nocopyright

\usepackage{newfloat}
\usepackage{listings}
\DeclareCaptionStyle{ruled}{labelfont=normalfont,labelsep=colon,strut=off} 
\floatstyle{ruled}
\newfloat{listing}{tb}{lst}{}
\floatname{listing}{Listing}
\usepackage{amsmath,amsthm,amssymb}
\usepackage{mathtools}
\usepackage{float}
\usepackage[capitalise]{cleveref}
\usepackage{booktabs}
\usepackage{boxedminipage}
\usepackage[textsize=tiny]{todonotes}
\usepackage{subcaption}

\usepackage{xspace}
\usepackage{amsmath}

\newcommand{\ignore}[1]{{}}
\newcommand{\tuple}[1]{\langle #1 \rangle}

\newtheorem{thm}{Theorem}
\newtheorem{definition}{Definition}

\renewcommand{\implies}{\rightarrow}

\newcommand{\mathdefault}[1][]{}

\newtheorem*{example*}{Example}

\usepackage[noend]{algpseudocode}
\usepackage{algorithm}
\algrenewcommand\alglinenumber[1]{\tiny #1}
\usepackage{etoolbox}
\AtBeginEnvironment{algorithmic}{\fontsize{8}{8}\selectfont}

\usepackage{tikz}
\usetikzlibrary{decorations.pathreplacing, calc, positioning, arrows.meta, automata}

\makeatletter
\newbox\dottedarrow@box
\setbox\dottedarrow@box\hbox
  {
    \begin{tikzpicture}
      \draw[dotted,->] (0,0) -- (1.5em,0);
    \end{tikzpicture}%
  }
\newcommand*\dottedarrow
  {\relax\ifmmode\expandafter\dottedarrow@m\else\expandafter\dottedarrow@t\fi}
\newcommand*\dottedarrow@t[1][1.5em]
  {\resizebox{#1}{!}{\raisebox{.5ex}{\usebox\dottedarrow@box}}}
\newcommand*\dottedarrow@m[1][]
  {%
    \if\relax\detokenize{#1}\relax
      \mathchoice
        {\dottedarrow@t}
        {\dottedarrow@t}
        {\dottedarrow@t[1.1em]}
        {\dottedarrow@t[0.9em]}%
    \else
      \dottedarrow@t[#1]%
    \fi
  }
\makeatother

\makeatletter
\newbox\solidarrow@box
\setbox\solidarrow@box\hbox
  {%
    \begin{tikzpicture}
      \draw[->] (0,0) -- (1.5em,0);
    \end{tikzpicture}%
  }
\newcommand*\solidarrow
  {\relax\ifmmode\expandafter\solidarrow@m\else\expandafter\solidarrow@t\fi}
\newcommand*\solidarrow@t[1][1.5em]
  {\resizebox{#1}{!}{\raisebox{.5ex}{\usebox\solidarrow@box}}}
\newcommand*\solidarrow@m[1][]
  {%
    \if\relax\detokenize{#1}\relax
      \mathchoice
        {\solidarrow@t}
        {\solidarrow@t}
        {\solidarrow@t[1.1em]}
        {\solidarrow@t[0.9em]}%
    \else
      \solidarrow@t[#1]%
    \fi
  }
\makeatother

\usepackage{array}

\usepackage{enumitem}
\setitemize{noitemsep,topsep=0pt,parsep=0pt,partopsep=0pt,leftmargin=10pt,itemindent=0pt}
\setenumerate{noitemsep,topsep=0pt,parsep=0pt,partopsep=0pt,leftmargin=10pt,itemindent=0pt}

\makeatletter
\newcommand{\pushright}[1]{\ifmeasuring@#1\else\omit\hfill$\displaystyle#1$\fi\ignorespaces}
\newcommand{\pushleft}[1]{\ifmeasuring@#1\else\omit$\displaystyle#1$\hfill\fi\ignorespaces}
\makeatother

\title{Platform-Aware Mission Planning\thanks{This paper has been accepted in the proceedings of the International Conference of Automated Planning and Scheduling (ICAPS) 2025.}}

\author {
    Stefan Panjkovic \textsuperscript{\rm 1,2},
    Alessandro Cimatti \textsuperscript{\rm 1},
    Andrea Micheli \textsuperscript{\rm 1},
    Stefano Tonetta \textsuperscript{\rm 1}
}
\affiliations {
    \textsuperscript{\rm 1}Fondazione Bruno Kessler, Trento, Italy\\
    \textsuperscript{\rm 2}University of Trento, Italy\\
    \texttt{\{spanjkovic,cimatti,amicheli,tonettas\}@fbk.eu}
}

\usepackage{bibentry}

\begin{document}

\maketitle

\begin{abstract}
    Planning for autonomous systems typically requires reasoning with models at different levels of abstraction, and the harmonization of two competing sets of objectives: high-level mission goals that refer to an interaction of the system with the external environment, and low-level platform constraints that aim to preserve the integrity and the correct interaction of the subsystems. The complicated interplay between these two models makes it very hard to reason on the system as a whole, especially when the objective is to find plans with robustness guarantees, considering the non-deterministic behavior of the lower layers of the system.

    In this paper, we introduce the problem of Platform-Aware Mission Planning (PAMP), addressing it in the setting of temporal durative actions. The PAMP problem differs from standard temporal planning for its \emph{exists-forall} nature: the high-level plan dealing with mission goals is required to satisfy safety and executability constraints, for all the possible non-deterministic executions of the low-level model of the platform and the environment.
    We propose two approaches for solving PAMP. The first baseline approach amalgamates the mission and platform levels, while the second is based on an abstraction-refinement loop that leverages the combination of a planner and a verification engine.
    We prove the soundness and completeness of the proposed approaches and validate them experimentally, demonstrating the importance of heterogeneous modeling and the superiority of the technique based on abstraction-refinement.
\end{abstract}

\section{Introduction}



A commonly employed architecture to realize autonomous systems consists in using automated planning for the synthesis of plans to achieve given mission goals, which are then executed on the system's platform (and the environment the system operates in). This separation of concerns allows the planner to reason on a high-level model, disregarding details and non-deterministic behaviors of the platform.
However, in safety-critical applications, the mission objectives (also called ``science objectives'', in the space domain) often conflict with the safety constraints dictated by the platform and the environment (also called "engineering constraints").
The former are easily expressible as goals in the planning problem and are existential in nature (i.e., one needs to find any plan achieving the goals), whereas safety constraints are universally quantified over all the  possible executions of the given plan (i.e., one needs to prevent any execution violating the safety).
To represent this scenario and to offer reasoning guarantees on the autonomous system as a whole, we need to model \emph{both} the planning objectives and the possible platform behaviors, formalizing the link between these levels.



In this paper, we aim at finding solution plans achieving high-level mission goals and offering formal robustness guarantees on the execution of such plans on the platform. We formalize and tackle the ``Platform-Aware Mission Planning'' (PAMP) problem, which aims at finding a plan guaranteed to achieve the mission objectives, such that all the possible evolutions of the platform controlled by the plan satisfy a set of safety and executability properties, taking into account both the flexibility of the execution platform (that is, the possible choices the platform can operate while obeying a given plan) and the non-determinism from the environment.
Concretely, we propose a formal framework, in which a high-level
temporal planning representation is coupled with a low-level
description of the platform that executes the generated plans. We use
a standard temporal planning model adapted from
\citeauthor{gigante2022} (\citeyear{gigante2022}) and a timed
automaton \cite{timed_automata} to represent the platform level. This formal framework uses existing models and is
thus easily instantiated in practice; moreover, it mirrors the
architecture of real-world autonomous systems \cite{threelayerarch}.

We propose two techniques to solve the PAMP problem. First, we develop a baseline ``amalgamated'' approach grounded in Satisfiability Modulo Theory (SMT) \cite{smt}: we combine a standard encoding for temporal planning with a novel encoding for the executability and safety of a symbolic temporal plan. The resulting formula is quantified ($\exists\forall$), mirroring the intuitive quantifier alternation in the problem definition.
Second, we propose a much more efficient approach exploiting the subdivision of the framework in two layers of abstraction. The technique uses a state-of-the-art heuristic search temporal planner to generate candidate plans, which are then checked for safety and executability, explaining the conflicts as sequences of events that the planner is required to avoid for subsequent candidates. This check employs a specialized version of the amalgamated SMT encoding.
We formally prove the soundness and completeness of the proposed approaches and we develop two scalable case-studies to empirically evaluate them, showing their empirical effectiveness.

\section{Background}


\paragraph{Temporal Planning}

We start by defining the syntax of a temporal planning problem: we
adapt the formal model used by \citeauthor{gigante2022}
(\citeyear{gigante2022}), which is quite close to
PDDL 2.1 level 3 \cite{pddl21}.

\begin{definition}[Temporal Planning Problem]
A temporal planning problem $\Pi$ is a tuple $\tuple{P, A, I, G}$,
where $P$ is a set of propositions, $A$ is a set of durative actions,
$I: P \rightarrow \{ \top, \bot \}$ is the initial state and $G \subseteq P$ is the goal
condition. A snap (instantaneous) action is a tuple $h = \tuple{\text{pre}(h), \text{eff}^+(h), \text{eff}^-(h)}$, where $\text{pre}(h) \subseteq P$ is the set of preconditions and $\text{eff}^+(h), \text{eff}^-(h) \subseteq P$ are two disjoint sets of propositions, called the positive and negative effects of $h$, respectively. We write $\text{eff}(h)$ for $\text{eff}^+(h) \cup \text{eff}^-(h)$. A durative action $a \in A$ is a tuple $\tuple{a_\vdash, a_\dashv, \text{pre}^{\leftrightarrow}(a), [L_a, U_a]}$, where $a_\vdash$ and $a_\dashv$ are the start and end snap actions, respectively, $\text{pre}^{\leftrightarrow}(a) \subseteq P$ is the over-all condition, and $L_a \in \mathbb{Q}_{>0}$ and $U_a \in \mathbb{Q}_{>0} \cup \{ \infty \}$ are the bounds on the action duration.
\end{definition}

A (time-triggered) plan is defined as a set of triples, each specifying an action, its starting time and its duration.

\begin{definition}[Plan]
Let $\Pi = \tuple{P, A, I, G}$ be a temporal planning problem. A plan for $\Pi$ is a set of tuples $\pi = \{ \tuple{a_1, t_1, d_1}, \ldots, \tuple{a_n, t_n, d_n} \}$, where, for each $1 \leq i \leq n$, $a_i \in A$ is a durative action, $t_i \in \mathbb{Q}_{\geq 0}$ is its start time, and $d_i \in \mathbb{Q}_{>0}$ is its duration.
\end{definition}

We will call \emph{length} of a time-triggered plan $\pi$ (denoted with $\vert \pi \vert$) the number of snap actions in $\pi$ (i.e. twice the number of durative actions).

A time-triggered plan $\pi$ is a solution plan for the problem $\Pi$
if, starting from the initial state $I$, each durative action in the plan can be applied at the specified
time with the given duration (the preconditions of its start and end
snap actions are true at the start and at the end of the action
respectively), and if by applying all the effects a final state is
reached after the end of the last action in which the goal condition
is satisfied. The formal semantics is presented in \cite{gigante2022}, which we omit here for the sake of brevity.

We assume a semantics without self-overlapping of actions \cite{gigante2022}, which makes the temporal planning problem decidable: it is not possible for two instances of the same ground action to overlap in time.

\begin{definition}[Action self-overlapping]
A plan $\{ \tuple{a_1, t_1, d_1}, \ldots, \tuple{a_n, t_n, d_n} \}$ is without self-overlapping if there exist no $i, j \in \{ 1, \ldots, n \}$ such that $a_i = a_j$ and $t_i \leq t_j < t_i + d_i$.
\end{definition}

This formal model of temporal planning is simplified with respect to concrete planning languages (e.g. for the sake of simplicity we only defined the ground model, while most languages allow a first-order lifted representation for compactness), but it already achieves the full computational complexity of very expressive languages such as ANML \cite{gigante_kr}.

\paragraph{Timed Automata}

Here, we recall the standard definitions.

\begin{definition}[Clock constraints]
Let $\mathcal{X}$ be a finite set of elements called clocks. A clock constraint is a conjunctive formula of atomic constraints of the form $x \sim n$ or $x-y \sim n$, where $x, y \in \mathcal{X}$, $\sim \in \{ \leq, <, =, >, \geq \}$ and $n \in \mathbb{N}$. We use $\mathcal{C}(\mathcal{X})$ to denote the set of clock constraints on $\mathcal{X}$.
\end{definition}

A Timed Automaton generalizes finite-state automata by means of clock variables that can be reset and track the advancement of time.

\begin{definition}[Timed Automaton]
A Timed Automaton (TA) is a tuple $\mathcal{T} = \tuple{\Sigma, \mathcal{L}, l_0, \mathcal{X}, \Delta, \text{Inv}}$, where:
\begin{itemize}
    \item $\Sigma$ is the alphabet;
    \item $\mathcal{L}$ is a finite set of locations;
    \item $l_0 \in \mathcal{L}$ is the initial location;
    \item $\mathcal{X}$ is a finite set of clocks;
    \item $\Delta \subseteq \mathcal{L} \times \mathcal{C}(\mathcal{X}) \times \Sigma \times 2^\mathcal{X} \times \mathcal{L}$ is the transition relation;
    \item $\text{Inv}: \mathcal{L} \rightarrow \mathcal{C}(\mathcal{X})$ maps each location to its invariant.
\end{itemize}
We will write $l \xrightarrow{g, a, r} l^\prime$ when $\tuple{l, g, a, r, l^\prime} \in \Delta$.
\end{definition}

\begin{definition}[State of TAs]
Given a TA $\mathcal{T} = \tuple{\Sigma, \mathcal{L}, l_0, \mathcal{X}, \Delta, \text{Inv}}$, a state of $\mathcal{T}$ is a pair $\tuple{l, u}$, where $l \in \mathcal{L}$ and $u: \mathcal{X} \rightarrow \mathbb{R}_{\geq 0}$ is a clock assignment.
\end{definition}

\noindent
We use $u \models g$ to mean that the clock values denoted by $u$ satisfy the guard $g \in \mathcal{C}(\mathcal{X})$. For $d \in \mathbb{R}_{\geq 0}$, we use $u+d$ to denote the clock assignment that maps all clocks $c \in \mathcal{X}$ to $u(c)+d$. For $r \subseteq \mathcal{X}$, we use $[r \rightarrow 0]u$ to denote the clock assignment that maps all $c \in r$ to 0, and all $c \in \mathcal{X}\! \setminus r$ to $u(c)$.

\begin{definition}[Semantics of TAs]
The semantics of a TA is defined in terms of a transition system, with states of the form $\tuple{l, u}$ and transitions defined by the following rules:
\begin{itemize}
    \item $\tuple{l, u} \xrightarrow{d} \tuple{l, u+d}$ if $u \in \text{Inv}(l)$ and $(u+d) \in \text{Inv}(l)$, for $d \in \mathbb{R}_{\geq 0}$;
    \item $\tuple{l, u} \xrightarrow{a} \tuple{l^\prime, u^\prime}$ if $l \xrightarrow{g, a, r} l^\prime$, $u \models g$, $u^\prime = [r \rightarrow 0] u$ and $u^\prime \in \text{Inv}(l^\prime)$.
\end{itemize}
\end{definition}

\begin{definition}[Timed trace]
Let $\mathcal{T} = \tuple{\Sigma, \mathcal{L}, l_0, \mathcal{X}, \Delta, \text{Inv}}$ be a TA. A timed action is a pair $\tuple{t, a}$, where $t \in \mathbb{R}_{\geq 0}$ and $a \in \Sigma$. A timed trace is a (possibly infinite) sequence of timed actions $\xi = \tuple{\tuple{t_1, a_1}, \tuple{t_2, a_2}, \ldots, \tuple{t_i, a_i}, \ldots}$, where $t_i \leq t_{i+1}$ for all $i \geq 1$.
\end{definition}

\begin{definition}[Run of a TA]
The run of a TA $\mathcal{T} = \tuple{\Sigma, \mathcal{L}, l_0,
  \mathcal{X}, \Delta, \text{Inv}}$ with initial state $\tuple{l_0,
  u_0}$ over a timed trace $\xi = \tuple{\tuple{t_1, a_1}, \tuple{t_2,
    a_2}, \ldots}$ is the sequence of transitions
$\tuple{l_0, u_0} \xrightarrow{d_1} \xrightarrow{a_1} \tuple{l_1, u_1} \xrightarrow{d_2} \xrightarrow{a_2} \tuple{l_2, u_2} \ldots$,
where $d_1 = t_1$ and $d_i = t_i - t_{i-1}$ for all $i \geq 2$.
\end{definition}

\begin{figure*}[tb]
  \begin{subfigure}[b]{0.5\textwidth}
    \resizebox{\columnwidth}{!}{\begin{tikzpicture}[shorten >=1pt,node distance=4.5cm,on grid,>={Stealth[round]},
    every state/.style={draw=blue!50,very thick,fill=blue!20,minimum size=2cm,align=center,text width=2cm}]

    \node[state,initial] (N1) {\textsc{off}};
    \node[state] (N2) [below=of N1] {\textsc{p\_started}};
    \node[state] (N25) [right=of N2] {\textsc{w\_starting} \\ $c_W \leq 2$};
    \node[state] (N3) [right=of N25] {\textsc{w\_started} \\ $c_W \leq 20$};
    \node[state] (N35) [above=of N3] {\textsc{w\_resuming} \\ $c_W \leq 2$};
    \node[state] (N4) [right=of N3] {\textsc{w\_ended}};
    \node[state] (N6) [right=of N4] {\textsc{p\_ended}};
    \node[state] (N5) [above=of N6] {\textsc{bad}};
    \node[state] (N7) [above=of N4] {\textsc{c\_started} \\ $c_C \leq 2$};

    \path[->] (N1) edge node [left,align=center,text width=1.5cm] {$P_\vdash$ \\ $c_P \coloneqq 0$} (N2)
              (N2) edge node [above] {$W_\vdash$} node [below] {$c_W \coloneqq 0$} (N25)
              (N25) edge node [above] {$\tau$} (N3)
              (N3) edge node [below,align=center,text width=2cm] {$W_\dashv$ \\ $c_W == 20$ \\ $c \coloneqq 0$} (N4)
              (N4) edge node [below,align=center,text width=1.5cm] {$W_\vdash$ \\ $c \geq 10$ \\ $c_W \coloneqq 0$} (N35)
              (N35) edge node [left] {$\tau$} (N3)
              (N4) edge node [right,align=center,text width=1.5cm] {$\tau$ \\ $c_P > 50$} (N5)
              (N4) edge node [below,align=center,text width=1.5cm] {$P_\dashv$} (N6)
              (N4) edge node [right,align=center,text width=1.5cm] {$C_\vdash$ \\ $c_C \coloneqq 0$} (N7)
              (N7) edge [bend right=45] node [below=0.3cm,align=center,text width=2cm] {$C_\dashv$ \\ $c_C == 2$} (N2)
              (N5) edge node [right] {$P_\dashv$} (N6);
\end{tikzpicture}}
    \caption{\label{fig:factory_ta}}
  \end{subfigure}
  \hfill
  \begin{subfigure}[b]{0.42\textwidth}
    \resizebox{\columnwidth}{!}{\definecolor{dark-gray}{gray}{0.5}
\begin{tikzpicture}[>=latex, very thick]
    \node [anchor=west, draw=black, rectangle, text width=4cm, align=center] (w1) {Work [20]};
    \node [draw=dark-gray, rectangle, dashed, text width=5cm, align=center, right=0.4cm of w1] (c) {distance $\geq 10$};
    \node [draw=black, rectangle, text width=4cm, align=center, right=0.4cm of c] (w2) {Work [20]};
    \node [draw=black, rectangle, text width=15cm, align=center, below left=0.4cm and 0.3cm of w1, anchor=north west] (p) {Processing [55]};

    \draw [-latex] ($ (p.south west) + (-0.3, -0.4) $) -- ($ (p.south east) + (1, -0.4) $) node [anchor=north] {t};

    \draw [dotted] (p.south west) -- ($ (p.south west) + (0, -0.4) $) node [anchor=north] {$0$};
    \draw [dotted] (p.south east) -- ($ (p.south east) + (0, -0.4) $) node [anchor=north] {$55$};

    \draw [-latex, red] (p.north west) -- (w1.west);
    \draw [-latex, red] (w1.east) -- (c.west);
    \draw [-latex, red] (c.east) -- (w2.west);
    \draw [-latex, red] (w2.east) -- (p.north east);

\begin{scope}[yshift=-3cm]
    \node [anchor=west, draw=black, rectangle, text width=4cm, align=center] (w1) {Work [20]};
    \node [draw=black, rectangle, text width=4cm, align=center, right=3cm of w1] (w2) {Work [20]};
    \node [draw=black, rectangle, text width=12cm, align=center, below left=0.4cm and 0.3cm of w1, anchor=north west] (p) {Processing [45]};

    \draw [-latex] ($ (p.south west) + (-0.3, -0.4) $) -- ($ (p.south east) + (4, -0.4) $) node [anchor=north] {t};

    \draw [dotted] (p.south west) -- ($ (p.south west) + (0, -0.4) $) node [anchor=north] {$0$};
    \draw [dotted] (p.south east) -- ($ (p.south east) + (0, -0.4) $) node [anchor=north] {$45$};

    \draw [-latex, red] (p.north west) -- (w1.west);
    \draw [-latex, red] (w1.east) -- (w2.west);
    \draw [-latex, red] (w2.east) -- (p.north east);
\end{scope}

\begin{scope}[yshift=-6cm]
    \node [anchor=west, draw=black, rectangle, text width=4cm, align=center] (w1) {Work [20]};
    \node [draw=black, rectangle, text width=2cm, align=center, right=0.4cm of w1] (c) {Cool [2]};
    \node [draw=black, rectangle, text width=4cm, align=center, right=0.4cm of c] (w2) {Work [20]};
    \node [draw=black, rectangle, text width=12cm, align=center, below left=0.4cm and 0.3cm of w1, anchor=north west] (p) {Processing [45]};

    \draw [-latex] ($ (p.south west) + (-0.3, -0.4) $) -- ($ (p.south east) + (4, -0.4) $) node [anchor=north] {t};

    \draw [dotted] (p.south west) -- ($ (p.south west) + (0, -0.4) $) node [anchor=north] {$0$};
    \draw [dotted] (p.south east) -- ($ (p.south east) + (0, -0.4) $) node [anchor=north] {$45$};

    \draw [-latex, red] (p.north west) -- (w1.west);
    \draw [-latex, red] (w1.east) -- (c.west);
    \draw [-latex, red] (c.east) -- (w2.west);
    \draw [-latex, red] (w2.east) -- (p.north east);
\end{scope}

\draw (-1, -1) node [left] {\scalebox{3}{$\pi_1$}};
\draw (-1, -4) node [left] {\scalebox{3}{$\pi_2$}};
\draw (-1, -7) node [left] {\scalebox{3}{$\pi_3$}};

\end{tikzpicture}}
    \caption{\label{fig:factory_plans}}
  \end{subfigure}
  \caption{Running example TA platform model (\subref{fig:factory_ta}) and example plans (\subref{fig:factory_plans}). The first two plans violate safety ($\pi_1$) and executability ($\pi_2$) constraints, while the third one ($\pi_3$) is correct for all platform executions.}
  \label{fig:running-example}
\end{figure*}

\section{Problem definition}

In this section, we formalize our composite framework and the Platform-Aware Mission Planning (PAMP) problem.

In our framework, we consider an autonomous system architecture with
two layers of abstraction: a \emph{planning layer}, represented as
a temporal planning problem, describing the high-level
durative actions and a mission goal; and a
\emph{platform layer}, represented as a TA,
which describes the low-level details and internal actions of the platform that is controlled by the planner. We consider an interface between the two layers where each start and end event of an action of the planning problem is associated with a signal of the TA (a letter of its alphabet), and define the execution of a time-triggered plan by synchronizing the action start/end commands of the plan with transitions of the platform labeled with the corresponding events. In the time between two high-level commands, the platform can freely evolve by performing internal transitions and advancing time.

\begin{example*}
  Figures \ref{fig:factory_ta} and \ref{fig:factory_plans} show a small running example of the considered framework. An industrial process needs to be completed by starting a "Process" (P) action and applying in parallel two "Work" (W) actions. In the planning model, we have a Boolean variable \emph{processing} (initially false) and a bounded integer variable \emph{completed-steps} (initially 0).\footnote{For simplicity, we use a bounded integer variable to count the number of completed "Work" actions. Such a variable can be compiled in our planning model using unary or binary encodings.}
  The "Process" action sets the \emph{processing} variable to true at the start, and sets it back to false at the end. The "Work" action has an over-all condition requiring the \emph{processing} variable to be true during the entire duration of the action, and an end effect that increments the value of the \emph{completed-steps} variable by 1. There is also a "Cooldown" action, which does not have any effect on the variables of the planning model. The goal requires that the "Work" action is applied two times (i.e. $\text{\emph{completed-steps}} == 2$).
  At the platform layer, modeled by the TA shown in \cref{fig:factory_ta}, there are transitions with labels corresponding to the start/end events of the high-level durative actions (e.g. $W_\vdash$ corresponds to the start of the "Work" action), and internal transitions that the platform can take, which are not linked to high-level events (the transitions with label $\tau$ in this example). The TA encodes a low-level constraint between successive applications of the "Work" actions that is not modeled in the planning problem: when a "Work" action is performed (reaching the \textsc{w\_ended} location), a component becomes heated and needs to cool down before the next "Work" action can be applied, and this occurs either by waiting 10 time units (transition from \textsc{w\_ended} to \textsc{w\_resuming} with guard $c \geq 10$), or by explicitly applying a "Cooldown" action which cools the component after 2 time units (transitions to \textsc{c\_started} and \textsc{p\_started} with labels $C_\vdash$ and $C_\dashv$). Moreover, the process has a deadline of 50 time units, after which the platform can reach an undesirable state (transition from \textsc{w\_ended} to \textsc{bad} with guard $c_P > 50$).
\end{example*}

  We start by introducing the notion of states that are reachable by executing a plan on a TA. Intuitively, the states that can be reached by executing a sequence of timed snap actions $\rho$ on a TA $\mathcal{T}$, are all the states that belong to a run of $\mathcal{T}$ where all and only the snap actions in $\rho$ are applied, by taking the corresponding transitions at the times specified in $\rho$. We formally define this with a function $H$, that maps the snap actions of $\rho$ with steps in the run of $\mathcal{T}$ where the transitions with the corresponding labels are taken.

\begin{definition}[States reachable by plan execution]\label{def:reachable}
Let $\Pi = \tuple{P, A, I, G}$ be a temporal planning problem, and let $\mathcal{T} = \tuple{\Sigma, \mathcal{L}, l_0, \mathcal{X}, \Delta, \text{Inv}}$ be a TA such that $\tau^{a_\vdash}, \tau^{a_\dashv} \in \Sigma$, for all actions $a \in A$. Let $\rho = \tuple{(t_1, e_1), \ldots, (t_n, e_n)}$ be a (possibly empty) ordered sequence of timed snap actions of $\Pi$, where $t_i < t_{i+1}$ for all $i \in \{ 1, \ldots, n-1 \}$. A state $r_s$ is "reachable" by executing $\rho$ on $\mathcal{T}$ from the initial state $r_0 = \tuple{l_0, u_0}$ if and only if there exists a run $r_0 \xrightarrow{d_1} \xrightarrow{\sigma_1} \ldots \xrightarrow{d_k} \xrightarrow{\sigma_k} r_k$, with $0 \leq s \leq k$, and an injective function $H : \{ 0, 1, \ldots, n \} \rightarrow \{ 0, 1, \ldots, k \}$ with the following properties:
\begin{enumerate}
  \item $H(0) = 0$ (Required to handle the case with $\rho = \tuple{}$);
  \item for all $i \in \{ 1, \ldots, n \}$, for all $j \in \{ 1, \ldots, k \}$, if $H(i) = j$ then $\tau^{e_i} = \sigma_j$ and $t_i = \sum_{l=1}^j d_l$;
  \item for all $j \in \{ 1, \ldots, k \}$, if $j \not\in \text{Im}(H)$ then for all $e \in \{ a_\vdash, a_\dashv : a \in A \}$, $\sigma_j \neq \tau^e$.
\end{enumerate}
\end{definition}

We define analogously the set of states that are reachable
\emph{after} executing $\rho$ on $\mathcal{T}$ from the initial state
$r_0$, denoted by $\text{ReachableAfter}_\mathcal{T}(r_0, \rho)$
(while in the previous definition we consider all states $r_s$ along
the run, here we include in the set only the final state $r_k$).

In the running example, consider the sequence $\rho = \tuple{(0, P_\vdash), (1, W_\vdash), (21, W_\dashv)}$. Then we have that
\begin{align*}
  &\text{Reachable}_\mathcal{T}(\textsc{off}, \rho) = &&\{ \textsc{off}, \textsc{p\_started}, \textsc{w\_starting},\\
  & &&\textsc{w\_started}, \textsc{w\_ended}, \textsc{bad} \}\\
  &\text{ReachableAfter}_\mathcal{T}(\textsc{off}, &&\rho) = \{ \textsc{w\_ended}, \textsc{bad} \}
\end{align*}
For instance, $\textsc{bad} \in \text{ReachableAfter}_\mathcal{T}(\textsc{off}, \rho)$ since there exists the run $\textsc{off} \xrightarrow{d_1 = 0} \xrightarrow{\sigma_1 = P_\vdash} \textsc{p\_started} \xrightarrow{d_2 = 1} \xrightarrow{\sigma_2 = W_\vdash} \textsc{w\_starting} \xrightarrow{d_3 = 2} \xrightarrow{\sigma_3 = \tau} \textsc{w\_started} \xrightarrow{d_4 = 18} \xrightarrow{\sigma_4 = W_\dashv} \textsc{w\_ended} \xrightarrow{d_5 = 31} \xrightarrow{\sigma_5 = \tau} \textsc{bad}$ and the function $H$ s.t. $H(1) \!=\! 1$, $H(2) \!=\! 2$ and $H(3) \!=\! 4$, satisfying \cref{def:reachable}.

Next, we formalize the notion of executability of a time-triggered plan on the platform. Intuitively, we say that a time-triggered plan is executable on a TA if every snap action of the plan is applicable at the prescribed time, for any possible internal behavior of the platform, assuming that the platform applied all the previous commands of the plan. A snap action is applicable if a corresponding transition can be taken at the time specified in the plan.

Formally, given a state $\tuple{l, u}$ of $\mathcal{T}$, a snap action $a_{\vdash / \dashv}$ is \emph{applicable} in $\tuple{l, u}$ if and only if there exists a transition $l \xrightarrow{g, \tau^{a_{\vdash / \dashv}}, r} l^\prime$ such that $u \models g$ and $[r \rightarrow 0]u \in \text{Inv}(l^\prime)$.

For a time-triggered plan $\pi = \{ \tuple{a_1, t_1, d_1}, \ldots, \tuple{a_n, t_n, d_n} \}$, we indicate with $\rho^\pi = \tuple{(t_1^\prime, e_1), \ldots, (t_{2n}^\prime, e_{2n})}$ the ordered sequence of timed snap actions of $\pi$, with $t_i^\prime < t_{i+1}^\prime$ for all $i \in \{ 1, \ldots, 2n-1 \}$. For simplicity, we assume that all the valid plans of the considered planning problems do not contain simultaneous events, i.e. snap actions scheduled at the same time: since the semantics of TA is super-dense (multiple discrete steps can be taken at the same time in a specific order), in order to properly define and check the executability of a plan with simultaneous events for all platform behaviors, all the possible orderings for the sets of simultaneous events would need to be considered.
Given a sequence of timed snap actions $\rho = \tuple{(t_1, e_1), \ldots, (t_n, e_n)}$, we denote with $\rho_i = \tuple{(t_1, e_1), \ldots, (t_i, e_i)}$ the prefix obtained by considering the first $i \leq n$ timed snap actions. We denote with $\rho_0 = \tuple{}$ the empty sequence.

\begin{definition}[Time-triggered plan executability on TA]\label{def:executability}
Let $\Pi$ be a temporal planning problem and let $\mathcal{T}$ be a TA with initial state $r_0 = \tuple{l_0, u_0}$. Suppose that $\mathcal{T}$ has a global clock $\gamma$ that is not reset in any transition and has value 0 in the initial state. An ordered sequence of timed snap actions $\rho = \tuple{(t_1, e_1), \ldots, (t_n, e_n)}$ is executable on $\mathcal{T}$ if and only if for all $i \in \{ 0, \ldots, n-1 \}$, for all $r = \tuple{l, u} \in \text{ReachableAfter}_\mathcal{T}(r_0, \rho_i)$, if $u(\gamma) \!=\! t_{i+1}$ then $e_{i+1}$ is applicable in $r$. A time-triggered plan $\pi$ of $\Pi$ is executable on $\mathcal{T}$ if its sequence of timed snap actions $\rho^\pi$ is executable on $\mathcal{T}$.
\end{definition}

For example, the sequence $\rho = \{ (0, P_\vdash), (1, W_\vdash), (21, W_\dashv), (22, W_\vdash), (42, W_\dashv), (45, P_\dashv) \}$, which corresponds to the second plan in \cref{fig:factory_plans} is not executable on the TA of \cref{fig:factory_ta}, because it is possible to reach location $\textsc{w\_ended}$ with $\gamma = 22$ and $c = 1$ (this state belongs to $\text{ReachableAfter}_\mathcal{T}(r_0, \rho_3)$) and the transition with label $W_\vdash$ is not applicable (the guard $c \geq 10$ is false).

We formalize the notion of safety for a plan w.r.t a TA, given a set of bad states $B$, by requiring that all the states that can be reached by executing $\rho^\pi = \tuple{(t_1, e_1), \ldots, (t_n, e_n)}$, within time $t_n$, do not belong to $B$.

\begin{definition}[Plan safety w.r.t. TA]\label{def:safety}
Let $\Pi$ be a temporal planning problem and let $\mathcal{T}$ be a TA with initial state $r_0 = \tuple{l_0, u_0}$. Suppose that $\mathcal{T}$ has a global clock $\gamma$ that is not reset in any transition and has value 0 in the initial state. Let $B \subseteq \mathcal{L} \times \mathbb{R}^\mathcal{X}$ be a set of bad states for $\mathcal{T}$. An ordered sequence of timed snap actions $\rho = \tuple{(t_1, e_1), \ldots, (t_n, e_n)}$ is B-safe w.r.t. $\mathcal{T}$ if and only if for all states $r = \tuple{l, u} \in \text{Reachable}_\mathcal{T}(r_0, \rho)$ such that $u(\gamma) \leq t_n$, $r \notin B$. A time-triggered plan $\pi$ of $\Pi$ is B-safe w.r.t. $\mathcal{T}$ if its sequence of timed snap actions $\rho^\pi$ is B-safe w.r.t. $\mathcal{T}$.
\end{definition}

Consider the running example, and suppose that $B$ is the set of all states with location $\textsc{bad}$. The sequence $\rho = \{ (0, P_\vdash), (1, W_\vdash), (21, W_\dashv), (32, W_\vdash), (52, W_\dashv), (55, P_\dashv) \}$, which corresponds to the first plan in \cref{fig:factory_plans}, is not B-safe, because it is possible to reach location $\textsc{bad}$ with $\gamma = 53$ between the application of the last two snap actions $(52, W_\dashv)$ and $(55, P_\dashv)$ (this state belongs to $\text{Reachable}_\mathcal{T}(r_0, \rho)$).

We can now formally define the PAMP problem, where the objective is to find a solution plan for the planning problem, such that it is safe and executable for all the platform traces that are compliant with the plan.

\begin{definition}[PAMP]\label{def:platform_planning}
A Platform-Aware Mission Planning (PAMP) problem is a tuple $\Upsilon = \tuple{\Pi, \mathcal{T}, B}$, where $\Pi$ is a temporal planning problem, $\mathcal{T}$ is a TA, and $B \subseteq \mathcal{L} \times \mathbb{R}^\mathcal{X}$ is a set of bad states for $\mathcal{T}$. A solution for $\Upsilon$ is a plan $\pi$ such that:
(i) $\pi$ is a valid solution plan for $\Pi$;
(ii) $\pi$ is executable on $\mathcal{T}$;
(iii) $\pi$ is B-safe w.r.t. $\mathcal{T}$.
\end{definition}

The third plan in \cref{fig:factory_plans} is a solution for the example PAMP problem. The application of the "Cool" action between the two "Work" actions makes it fully executable and safe, since the $\textsc{bad}$ is unreachable ($c_P > 50$ remains false).

\section{Solution approaches}
\label{sec:solutions}

In this section, we propose two approaches for solving the PAMP problem.
%
%
We assume that a constant $k$ is given, which represents the maximum possible ratio between the length of a platform trace and the length of the executed plan. Hence, when considering a plan $\pi$ of length $L$ (the number of snap actions in the plan), we will analyze its safety and executability for platform traces of length up to $\kappa L$. It is reasonable to assume that such a constant exists and that it can be computed for a platform, as plans have a finite duration and in most practical systems only a finite number of transitions can be taken in a given time.

\paragraph{Encoding-based approach}

We will now describe our SMT encoding of the PAMP problem (\cref{fig:enc}). Consider a temporal planning
problem $\Pi$, a timed automaton $\mathcal{T}$ modeling the platform, a set of \emph{bad} states $B$, and a bound $h$ on the length of the plan. We assume that $\mathcal{T}$ contains a global clock $\gamma$ with initial value 0, that is never reset in any transition. The encoding represents two distinct traces: a \emph{plan trace} with $h$ timed steps, and a \emph{platform trace} with $\kappa h$ timed steps. In each step of a plan trace at most one snap action can be applied (as we discussed in the previous section).

\begin{figure*}[t]
    \centering
    \resizebox{.95\textwidth}{!}{
    \begin{boxedminipage}{18cm}
    \begin{align*}
      \small
      & \Phi_h : && \exists \vec{t}, \vec{a}, \vec{d}.  \textsc{PlanValid}_\Pi \left( \vec{t}, \vec{a}, \vec{d}, h \right) \land \forall \vec{l}, \vec{c}. &&  \bigwedge\limits_{i = 0}^{h-1} \left( \textsc{TraceValid}_\mathcal{T} \left( \vec{l}, \vec{c}, h, \kappa \right) \land \textsc{compliant}_\mathcal{T} \left( \vec{t}, \vec{a}, \vec{d}, \vec{l}, \vec{c}, i, \kappa \right) \implies \textsc{applicable}_\mathcal{T} \left( \vec{t}, \vec{a}, \vec{d}, \vec{l}, \vec{c}, i+1, \kappa \right) \right) \land \\
      & && && \left( \textsc{TraceValid}_\mathcal{T} \left( \vec{l}, \vec{c}, h, \kappa \right) \land \textsc{compliant}_\mathcal{T} \left( \vec{t}, \vec{a}, \vec{d}, \vec{l}, \vec{c}, h, \kappa \right) \implies \textsc{safety}_\mathcal{T} \left( \vec{l}, \vec{c}, B, h, \kappa \right) \right)
    \end{align*}
    \vspace{-5mm}
    \begin{align*}
        & \textsc{TraceValid}_\mathcal{T} \left( \vec{l}, \vec{c}, h, \kappa \right): && \textsc{Init}_\mathcal{T} \left( \vec{l}, \vec{c}, 1 \right) \land \bigwedge\limits_{i = 2}^{\kappa h} \textsc{Trans}_\mathcal{T} \left( \vec{l}, \vec{c}, i-1, i \right) \\
        & \textsc{compliant}_\mathcal{T} \left( \vec{t}, \vec{a}, \vec{d}, \vec{l}, \vec{c}, h, \kappa \right): && \bigwedge\limits_{a \in A} \bigwedge\limits_{i = 1}^h \left( a_i \implies \bigvee\limits_{j = 1}^{\kappa h} \left( \tau_j^{a_\vdash} \land \gamma_j = t_i \land \bigwedge\limits_{\substack{j^\prime \in \{ 1, \ldots, \kappa h \} \\ j^\prime \neq j}} \left( \gamma_{j^\prime} = t_i \implies \neg \tau_{j^\prime}^{a_\vdash} \right) \right) \right)\\
        & && \land \bigwedge\limits_{a \in A} \bigwedge\limits_{s = 1}^h \bigwedge\limits_{i = s+1}^h \left( \left( a_s \land t_s + d_s^a = t_i \right) \implies \bigvee\limits_{j = 1}^{\kappa h} \left( \tau_j^{a_\dashv} \land \gamma_j = t_i \land \bigwedge\limits_{\substack{j^\prime \in \{ 1, \ldots, \kappa h \} \\ j^\prime \neq j}} \left( \gamma_{j^\prime} = t_i \implies \neg \tau_{j^\prime}^{a_\dashv} \right) \right) \right) \land \\
        & && \bigwedge\limits_{a \in A} \bigwedge\limits_{i = 1}^{\kappa h} \left( \tau_i^{a_\vdash} \implies \bigvee\limits_{j = 1}^{h} \left( a_j \land t_j = \gamma_i \right) \right) \land \bigwedge\limits_{a \in A} \bigwedge\limits_{i = 1}^{\kappa h} \left( \tau_i^{a_\dashv} \implies \bigvee\limits_{s = 1}^h \bigvee\limits_{j = s+1}^h \left( a_s \land t_s + d_s^a = t_j \land t_j = \gamma_i \right) \right) \\
        & \textsc{applicable}_\mathcal{T} \left( \vec{t}, \vec{a}, \vec{d}, \vec{l}, \vec{c}, h, \kappa \right): && \bigwedge\limits_{a \in A} \left( a_h \implies \bigwedge\limits_{i = 1}^{\kappa h} \left( \left( \gamma_i = t_h \land \bigwedge\limits_{j=1}^{i-1} \left( \gamma_j < t_h \lor \neg \tau_j^{a_\vdash} \right) \right) \implies \bigvee\limits_{\delta = \tuple{l_i, g, \tau^{a_\vdash}, r, l_i^\prime} \in \Delta} \textsc{enabled}(\vec{l}, \vec{c}, \delta, i) \right) \right) \land \\
        & && \bigwedge\limits_{a \in A} \bigwedge\limits_{s = 1}^{h-1} \left( a_s \land t_s + d_s^a = t_h \implies \bigwedge\limits_{i = 1}^{\kappa h} \left( \left( \gamma_i = t_h \land \bigwedge\limits_{j=1}^{i-1} \left( \gamma_j < t_h \lor \neg \tau_j^{a_\dashv} \right) \right) \implies \bigvee\limits_{\delta = \tuple{l_i, g, \tau^{a_\dashv}, r, l_i^\prime} \in \Delta} \textsc{enabled}(\vec{l}, \vec{c}, \delta, i) \right) \right) \\
        & \textsc{safety}_\mathcal{T} \left( \vec{l}, \vec{c}, B, h, \kappa \right): && \bigwedge\limits_{i = 1}^{\kappa h} \left( \gamma_i \leq t_h \implies \neg \textsc{bad} \left( \vec{l}, \vec{c}, B, i \right) \right)
    \end{align*}
    \end{boxedminipage}
    }
    \caption{\label{fig:enc} The bounded encoding of the platform-aware planning problem.}
    \end{figure*}

We start by defining the variables of our encoding. For every step $i \in \{ 1, \ldots, h \}$ of the plan trace, we use the real variable $t_i$ to denote the time associated to step $i$; for every action $a$, we use the Boolean variable $a_i$ to denote whether the action $a$ is started at step $i$, and the real variable $d_i^a$ to represent the duration of action $a$ when started at step $i$. For every step $i \in \{ 1, \ldots, \kappa h \}$ of the platform trace, we use the variable $l_i$ to denote the location of $\mathcal{T}$ at step $i$; for every clock $c$, we use the variable $c_i$ to represent the value of clock $c$ at step $i$ (the value of the global clock $\gamma$ at step $i$ is $\gamma_i$); finally, for every action $a$, we use the Boolean variable $\tau_i^{a_\vdash}$ (respectively $\tau_i^{a_\dashv}$) to denote whether a transition with label $\tau^{a_\vdash}$ (respectively $\tau^{a_\dashv}$) will be taken by $\mathcal{T}$ at step $i$.

We will use the notation $\vec{t}$ to denote the set of variables of the form $t_i$, and analogously for $\vec{a}$, $\vec{d}$,
$\vec{l}$ and $\vec{c}$.

The formula $\Phi_h$ represents time-triggered plans of length up
to $h$ that satisfy \cref{def:platform_planning}: the plan must be a valid solution for $\Pi$ ($\textsc{PlanValid}_\Pi$); for all possible traces of $\mathcal{T}$, the plan must be executable, i.e. for all plan prefixes $i$ from $0$ to $h-1$, if a trace of $\mathcal{T}$ is valid ($\textsc{TraceValid}_\mathcal{T}$) and all the snap actions up to $i$ have been applied ($\textsc{compliant}_\mathcal{T}$), then the $(i+1)$-th snap action will be applicable at the prescribed time ($\textsc{applicable}_\mathcal{T}$); finally, for all possible traces of $\mathcal{T}$, all the states of $\mathcal{T}$ that can be visited by executing the plan do not intersect the set of bad states $B$, i.e. if a trace of $\mathcal{T}$ is valid and all the snap actions of the plan have been applied, then the safety property is satisfied ($\textsc{safety}_\mathcal{T}$).

The formula $\textsc{PlanValid}_\Pi$ is a standard bounded encoding of temporal planning \cite{shin2005processes}, which we omit for the sake of brevity. Similarly, the formula $\textsc{TraceValid}_\mathcal{T}$ is a standard unrolling of the transition relation of the timed automaton $\mathcal{T}$ up to step $\kappa h$, where we denote with $\textsc{init}_\mathcal{T}$ the formula for the initial state of $\mathcal{T}$, and with $\textsc{trans}_\mathcal{T}$ the transition relation of $\mathcal{T}$.

The formula $\textsc{compliant}_\mathcal{T}$ encodes the fact that a platform trace applies all the snap actions of a plan up to step $h$, and that no transition corresponding to a snap action is triggered at the wrong time or without the action being present in the plan (it characterizes the traces that appear in the definition of $\text{Reachable}_\mathcal{T}$): when an action $a$ is started at step $i$ ($a_i$), there exists a step $j$ in the platform trace where a transition with the corresponding label is triggered ($\tau_j^{a_\vdash}$), the value of the global clock corresponds to the time at which $a$ is started ($\gamma_j = t_i$), and there are no multiple occurrences of the corresponding label at the same time; the same applies for actions ending at step $i$, which were started at a previous step $s$ ($a_s \land t_s + d_s^a = t_i$); if a transition with label $a_\vdash$ is taken at step $i$ ($\tau_i^{a_\vdash}$), then there must exist a step $j$ in the plan trace at which $a$ is started and the times are the same ($a_j \land t_j = \gamma_i$), and similarly for transitions with label $a_\dashv$.

The applicability of snap actions is encoded by
$\textsc{applicable}_\mathcal{T}$. If the action $a$ is started at step $h$, then for all the steps $i$ in the platform trace where the
value of the global clock corresponds to the time at which $a$ is
started and the corresponding transition has not already been taken in a previous step $j$ at the current time ($\gamma_j < t_h \lor \neg \tau_j^{a_\vdash}$), there must exist a transition from the current location
$l_i$ with label $\tau^{a_\vdash}$ that is enabled ($\textsc{enabled}$ encodes the fact that the guard is true under the current clock evaluation and that the invariant of the reached location is true after the necessary clocks are reset). The applicability of ends is handled similarly.

Finally, the formula $\textsc{safety}_\mathcal{T}$ states that for all the steps in the platform trace that occur before the end of the plan ($\gamma_i \leq t_h$), the current state is not included in the set of bad states $B$ ($\textsc{BAD}$ encodes the set $B \subseteq \mathcal{L} \times \mathbb{R}^\mathcal{X}$).

The overall procedure ($\textsc{PAMP-Enc}$) builds the formulae $\Phi_h$ for increasing bounds $h$ and checks them with an SMT solver: if it returns UNSAT, then there is no safe and executable plan within bound $h$ and the bound is increased; if a model is returned, it corresponds to a solution to the PAMP problem, as it satisfies the planning constraints and the executability and safety properties for all platform traces.

We now show the soundness and completeness of the approach. Here we provide proof sketches for the theorems, while the full details are included in the additional material.

\begin{thm}[Soundness and completeness of encoding-based algorithm]\label{thm1}
For every PAMP problem $\Upsilon = \tuple{\Pi, \mathcal{T}, B}$ and every bound $\kappa$:
\begin{enumerate}
    \item if $\Call{PAMP-Enc}{\Pi, \mathcal{T}, B, \kappa}$ terminates and returns plan $\pi$, then $\pi$ is a valid solution for $\Upsilon$ (soundness);
    \item if there exists a solution for $\Upsilon$, then $\Call{PAMP-Enc}{\Pi, \mathcal{T}, B, \kappa}$ will eventually terminate and return a solution for $\Upsilon$ (completeness).
\end{enumerate}
\end{thm}
\begin{proof}
(Sketch)
\begin{enumerate}
    \item Suppose that the procedure returns plan $\pi$ at step $h$. Let $\mu$ be the model of $\Phi_h$ from which $\pi$ was extracted. We need to prove that $\pi$ satisfies \cref{def:platform_planning}:
    \begin{enumerate}
        \item First, we show that $\pi$ is a valid solution
          plan for $\Pi$. $\mu$ satisfies $\textsc{PlanValid}_\Pi(\vec{t}, \vec{a},
          \vec{d}, h)$, which is a standard bounded encoding of the
          temporal planning problem $\Pi$, hence the plan $\pi$,
          which is extracted from $\mu$, is a solution plan for $\Pi$
          of length up to $h$.
        \item Second, we need to prove that $\pi$ is executable on $\mathcal{T}$, i.e. it satisfies the requirements of \cref{def:executability}. Let $r_0 = \tuple{l_0, u_0}$ be the initial state of $\mathcal{T}$ and let $\rho^\pi = \tuple{(t_1, e_1), \ldots, (t_n, e_n)}$ be the ordered sequence of timed snap actions of $\pi$. Consider a prefix $i \in \{ 0, \ldots, n-1 \}$ of $\rho^\pi$ and a state $r = \tuple{l, u} \in \text{ReachableAfter}_\mathcal{T}(r_0, \rho_i^\pi)$ such that $u(\gamma) = t_{i+1}$ and $r$ is reachable within $\kappa h$ steps starting from $r_0$, i.e. there exists a run $r_0 \xrightarrow{d_1} \xrightarrow{\sigma_1} \ldots \xrightarrow{d_k} \xrightarrow{\sigma_k} r_k \equiv r$ with $k \leq \kappa h$.
        Then, it can be shown that $e_{i+1}$ is applicable in $r$: the run reaching $r$ together with the model $\mu$ satisfy $\textsc{TraceValid}_\mathcal{T}(\vec{l}, \vec{c}, h, \kappa)$ and $\textsc{Compliant}_\mathcal{T}(\vec{t}, \vec{a}, \vec{d}, \vec{l}, \vec{c}, i, \kappa)$; from the encoding $\Phi_h$ this implies that $\textsc{Applicable}_\mathcal{T}(\vec{t}, \vec{a}, \vec{d}, \vec{l}, \vec{c}, i+1, \kappa)$ is satisfied, and this implies that $e_{i+1}$ is applicable in $r$ (full details in the additional material). Therefore, by \cref{def:executability}, $\pi$ is executable on $\mathcal{T}$.
        \item Third, we need to prove that $\pi$ is $B$-safe w.r.t. to $\mathcal{T}$, i.e. it satisfies the requirements of \cref{def:safety}. Let $r_0 = \tuple{l_0, u_0}$ be the initial state of $\mathcal{T}$ and let $\rho^\pi = \tuple{(t_1, e_1), \ldots, (t_n, e_n)}$ be the ordered sequence of timed snap actions of $\pi$. Consider a state $r = \tuple{l, u} \in \text{Reachable}_\mathcal{T}(r_0, \rho^\pi)$ such that $u(\gamma) \leq t_n$ and $r$ is reachable within $\kappa h$ steps starting from $r_0$, i.e. there exists a run $r_0 \xrightarrow{d_1} \xrightarrow{\sigma_1} \ldots \xrightarrow{d_k} \xrightarrow{\sigma_k} r_k$, with $k \leq \kappa h$ and $r = r_i$ for some $i \in \{ 0, \ldots, k \}$.
        Then, it can be shown that $r \not\in B$: the run reaching $r$ together with the model $\mu$ satisfy $\textsc{TraceValid}_\mathcal{T}(\vec{l}, \vec{c}, h, \kappa)$ and $\textsc{Compliant}_\mathcal{T}(\vec{t}, \vec{a}, \vec{d}, \vec{l}, \vec{c}, h, \kappa)$; from the encoding $\Phi_h$ this implies that $\textsc{Safety}_\mathcal{T}(\vec{l}, \vec{c}, B, h, \kappa)$ is satisfied, and this implies that $r \not\in B$. 
    \end{enumerate}
    \item Let $\pi$ be the solution plan for $\Upsilon$ that exists by assumption. Since $\textsc{PlanValid}_\Pi(\vec{t}, \vec{a}, \vec{d}, h)$ is a standard bounded encoding of the temporal planning problem $\Pi$ with completeness guarantees, there exists a step $h$ for which there is a model $\mu \models \textsc{PlanValid}_\Pi(\vec{t}, \vec{a}, \vec{d}, h)$ such that $\pi$ can be extracted from it. We can then show that $\mu \models \Phi_h$, which implies that $\Call{PAMP-Enc}{\Pi, \mathcal{T}, B, \kappa}$ terminates at a step lower or equal than $h$ (because $\mu \models \Phi_h$).
    \qedhere
\end{enumerate}
\end{proof}


\paragraph{Abstraction-refinement approach}

In our second approach, based on an abstraction-refinement loop, we consider the planning problem and the validation problem separately: a temporal planner generates solution plans considering only the planning problem, and then the produced candidate plans are checked for executability and safety at the platform layer. Since we are considering time explicitly, it is not feasible to exclude single time-triggered plans at each validation check, as the planner, that is not aware of the platform constraints, can in most cases just slightly change the timing of the actions and the same problem would occur at the platform layer. Instead, at each failed validation check we want to exclude classes of plans, by determining that a certain sequence of discrete choices is infeasible, for any possible scheduling of the chosen snap actions.

For solving the planning problem, we rely on the $\textsc{Tamer}$
temporal planner \cite{aaai20}, which is a sound and complete approach
for temporal planning that is able to return plans expressed as Simple
Temporal Networks (STN) \cite{stn}: a returned solution $\pi_\text{STN}$
is characterized by a fixed ordering of snap actions $e_1, \ldots, e_n
\leftarrow \Call{Path}{\pi_\text{STN}}$, with each snap action $e_i$ associated to a symbolic time $t_i$. The times are ordered increasingly, with additional constraints between pairs of start/end snap actions representing the duration constraints of the corresponding durative actions. The planning algorithm implements an explicit-state heuristic-search approach, that works by exploring all the possible ordered sequences of snap actions, and updating in each state a STN whenever a new snap action is added to the sequence. If the set of STN constraints of a state becomes infeasible, then it can be pruned, as it means that the chosen sequence of discrete events cannot be scheduled while respecting the temporal constraints of the problem. If a goal state is reached, then all the time-triggered plans that satisfy the STN constraints of that state are valid solution plans, and a specific solution can be extracted by solving the constraints.

The main idea of our approach is to validate on the platform the set of STN constraints $\pi_\text{STN}$ produced by the planner, using an encoding similar to the one of the previous approach (\cref{fig:enc}). If there exists a solution to the STN constraints, that satisfies the executability and safety notions for all the platform traces, then it corresponds to an answer to the PAMP problem. Otherwise, we determine the shortest prefix $e_1, \ldots, e_i$ of the sequence of snap actions $\Call{Path}{\pi_\text{STN}}$, such that by considering only the STN constraints of $e_1, \ldots, e_i$, there does not exist a way to schedule them while guaranteeing executability and safety for all platform traces. If such a prefix is found, it can be learned by the planner, and all the states that are found during exploration whose path starts with a learned prefix can be pruned.

The overall procedure is detailed in \cref{alg:sol2}. The planning problem $\Pi$ is solved, and we obtain a set of solution plans $\pi_\text{STN}$, characterized by a fixed order of snap actions $e_1, \ldots, e_n \leftarrow \Call{Path}{\pi_\text{STN}}$, together with a set of temporal constraints between their associated times $t_1, \ldots, t_n$. The solution $\pi_\text{STN}$ is then passed to $\textsc{Check}$, together with the set of bad states $B$. The $\textsc{Check}$ procedure iterates over all the prefixes $i \in \{ 1, \ldots, n \}$, and builds the formula $\psi_\pi^i$, which is the subset of constraints of $\pi_\text{STN}$ considering only $t_1, \ldots, t_i$ ($[\pi_\text{STN}]_i$ is the conjunction of all the constraints containing $t_i$ and one of the times in $\{ t_1, \ldots, t_{i-1} \}$). The constraints $\psi_\pi^i$ are then used to produce the formula $\Phi_i$, which is an encoding of the PAMP problem, considering only candidate plans represented by $\psi_\pi^i$: in the formula of \cref{fig:enc}, $\textsc{PlanValid}_\Pi$ is replaced with $\psi_\pi^i$, while the forall formula is simplified considering the specific discrete choices $e_1, \ldots, e_i$ that are made by each plan represented by $\psi_\pi^i$ (for each step $j \in \{ 1, \ldots, i \}$, the truth value of all the variables $a_i$ is known and can be substituted in the formula). The formula $\Phi_i$ is then provided to an SMT solver: if it is unsatisfiable, then we can deduce that the prefix $e_1, \ldots, e_i$ is not valid for the platform, for any possible scheduling of the snap actions that respects the planning constraints, and therefore this path can be ``learned'' by the planner and used for pruning; if it is satisfiable, then the next prefix can be considered, and if the whole plan was being considered then a final solution can be extracted from such a model.

\begin{thm}[Soundness and completeness of abstraction-refinement algorithm]
For every PAMP problem $\Upsilon = \tuple{\Pi, \mathcal{T}, B}$ and every bound $\kappa$:
\begin{enumerate}
    \item if $\Call{PAMP-Ref}{\Pi, \mathcal{T}, B, \kappa}$ terminates and returns plan $\pi$, then $\pi$ is a valid solution for $\Upsilon$;
    \item if there exists a solution for $\Upsilon$, then $\Call{PAMP-Ref}{\Pi, \mathcal{T}, B, \kappa}$ will eventually terminate and return a solution for $\Upsilon$.
\end{enumerate}
\end{thm}
\begin{proof}
(Sketch)
\begin{enumerate}
    \item Suppose that the procedure returns plan $\pi$. $\pi$ is a valid solution to the planning problem $\Pi$, because of the soundness of the $\textsc{Tamer}$ planner \cite{aaai20} and the correctness of the encoding of the STN constraints (that replace the formula $\textsc{PlanValid}_\Pi(\vec{t}, \vec{a}, \vec{d}, h)$ in the $\Phi_h$ encoding). The plan then satisfies an analogous encoding of the forall subformula of $\Phi_h$ (which is simplied taking into account the specific action choices made by $\textsc{Tamer}$), and this guarantees the executability and safety properties (the proof follows the same reasoning of soundness in \cref{thm1}).
    \item If a solution $\pi$ to $\Upsilon$ exists, then because of the completeness of the $\textsc{Tamer}$ planner and the fact that the excluded prefixes do not satisfy the $\Phi_h$ encoding, the plan $\pi$ will be eventually returned by $\textsc{Tamer}$ (unless a different solution is found earlier for $\Upsilon$). The plan then satisfies an analogous encoding of the forall subformula of $\Phi_h$, which confirms the executability and safety properties of $\pi$ (the proof follows the soundness in \cref{thm1}). \qedhere
\end{enumerate}
\end{proof}

\begin{algorithm}[tb]
    \begin{algorithmic}[1]
        \Procedure{PAMP-Ref}{$\Pi$, $\mathcal{T}$, $B$, $\kappa$}
            \State $\text{bad\_prefixes} = \{\}$
            \While{True}
                \State $\pi_\text{STN} \leftarrow \Call{Plan}{\Pi, \text{bad\_prefixes}}$
                \State $\text{pass}, \pi \leftarrow \Call{Check}{\mathcal{T}, \pi_\text{STN}, B}$
                \If{pass}
                    \Return $\pi$
                \Else
                    \State $\text{bad\_prefixes} \leftarrow \text{bad\_prefixes} \cup \pi$
                \EndIf
            \EndWhile
        \EndProcedure
        \Procedure{Check}{$\mathcal{T}$, $\pi_\text{STN}$, $B$}
            \State $e_1, \ldots, e_n \leftarrow \Call{Path}{\pi_\text{STN}}$
            \State $\psi_\pi^0 \leftarrow \top$
            \For{$i = 1 \text{ to } n$}
                \State $\psi_\pi^i \leftarrow \psi_\pi^{i-1} \land \left[ \pi_\text{STN} \right]_i$
                \State $\Phi_i \leftarrow \Call{Encode}{\psi_\pi^i, \mathcal{T}, B, i, \kappa}$
                \State $\mu \leftarrow \Call{Solve}{\Phi_i}$
                \If{$\mu \text{\ is UNSAT}$}
                    \Return $\bot, (e_1, \ldots, e_i)$
                \Else
                    \If{$i == n$}
                        \Return $\top, \Call{ExtractPlan}{\mu}$
                    \EndIf
                \EndIf
            \EndFor
        \EndProcedure
    \end{algorithmic}
    \caption{\label{alg:sol2}Abstraction-refinement algorithm}
\end{algorithm}

\section{Related Work}

In this paper, we define the PAMP problem, aiming at offering \emph{guarantees} at system-level on the execution behaviors of the plan. To the best of our knowledge, this problem is novel. There is a wide literature on how to handle execution of plans in autonomous systems \cite{rosplan,bdi-roveri}, but here the point is to model the interaction between a planner and the underlying platform and construct plans that are provably safe by construction.

\citeauthor{omcare} (\citeyear{omcare}) propose a formal framework for on-board autonomy that relies on symbolic model-based reasoning, and integrates plan generation, plan execution, monitoring, fault detection identification and recovery, and run-time diagnosis functionalities. The controlled system is modeled as a finite-state non-deterministic planning problem enriched with resource estimation functions. In our case, we focus on a timed model for the system and the platform modeling is much richer than the estimation functions: we allow for generic timed automata models.


\begin{figure*}[tb]
    \centering
    \begin{subfigure}[b]{0.18\textwidth}
        \centering
        \resizebox{\textwidth}{!}{
            \begin{tabular}{lrr}
                \toprule
                Algorithm & $\textsc{Bmc}$ & $\textsc{Ref}$ \\
                Domain &  &  \\
                \midrule
                Factory1-k2 & 4 & 9 \\
                Factory1-k3 & 1 & 2 \\
                Factory1-k4 & 0 & 0 \\
                Factory1-k5 & 0 & 0 \\
                Factory2-k2 & 10 & 10 \\
                Factory2-k3 & 3 & 6 \\
                Factory2-k4 & 0 & 3 \\
                Factory2-k5 & 0 & 0 \\
                Rover-k2 & 35 & 44 \\
                Rover-k3 & 18 & 25 \\
                Rover-k4 & 8 & 18 \\
                Rover-k5 & 8 & 13 \\
                Total & 87 & 130 \\
                \bottomrule
                \end{tabular}
        }
        \caption{\label{fig:coverage}}
    \end{subfigure}
    \quad
    \begin{subfigure}[b]{0.41\textwidth}
        \centering
        \resizebox{\textwidth}{!}{\input{expeval-final/plots/cactus-time-legend-all.pgf}}
        \caption{\label{fig:cactus}}
    \end{subfigure}
    \quad
        \begin{subfigure}[b]{0.33\textwidth}
        \centering
        \resizebox{\textwidth}{!}{\input{expeval-final/plots/sat-scatter-platform-bmc-platform-refinement.pgf}}
        \caption{\label{fig:scatter}}
    \end{subfigure}
    \caption{Experimental results: coverage table (\subref{fig:coverage}), cactus (\subref{fig:cactus}) and scatter (\subref{fig:scatter}) plots. The $k$ values represent the bound on platform traces w.r.t. plan lengths. $\textsc{Bmc}$ is the encoding-based approach, while $\textsc{Ref}$ is the algorithm based on abstraction-refinement.}
    \label{fig:plots}
\end{figure*}

\citeauthor{viehmann2021} (\citeyear{viehmann2021}) assume that an abstract sequential plan is given and model the platform layer as a timed automaton. A set of constraints expressed in Metric Temporal Logic (MTL) define the relationship between the two layers. The paper is limited to the problem of checking if there is one execution of the platform satisfying a given plan (it is an $\exists\exists$ quantification), while in our case, we provide a formal model for the plan generation part and tackle a plan generation problem that requires a universally-quantified validation on the platform ($\exists\forall$). Moreover, we differ on the interface between the high and low abstraction layers: we use the labels in the timed automata to model ``commands'' that are sent by the plan, while \citeauthor{viehmann2021} use MTL formulae to constrain the possible traces. Finally, we also consider executability constraints.

The problem that we address is also strongly related to conformant planning \cite{traverso-book}, but in a temporal setting and for a model of the actions given by the timed automaton in the platform level. In fact, we are looking for a plan that is capable of succeeding (in both the planning goal and the safety and executability constraints) irrespectively of the non-determinism of the platform. We are not aware of any paper concerning conformant temporal planning, but our encodings into SMT resemble the approach in \cite{cimatti-conformant} for finite-state conformant planning. A key difference with respect to conformant planning is that in our case the non-determinism originates from the platform and can involve multiple steps not visible at planning level, whereas in conformant planning one assumes that either the initial state or action effects are non-deterministic (i.e., there is a ``lockstep'' between the planning choices and the nondeterministic outcomes).


\section{Experimental Evaluation}

We developed a solver written in Python based on pySMT \cite{pysmt} implementing both the presented approaches. The solvers accept temporal planning problems written in either PDDL 2.1 or ANML, and platform models written in timed SMV \cite{magnago}, which allows to model TAs in a symbolic setting.
%
%
We experimentally evaluated both approaches on two novel sets of benchmarks, $\textsc{Rover}$ and $\textsc{Factory}$, which are both available in the additional material.
In $\textsc{Rover}$, there are $n$ locations $l_0, \ldots, l_{n-1}$ connected by edges, and a robot which is initially at location $l_0$. The robot can move between consecutive locations in 1 unit of time, while moving between non-consecutive locations takes 100 units of time. The robot can also communicate at each location. The goal of the planning problem is to communicate while at certain locations, and reach location $l_{n-1}$ in the end. At the platform layer, modeled as a network of TAs, there is a task component (synchronized with the high-level communicate action) that controls a communication component: when a message is sent, the communication component moves to a \emph{standby} location if no other message is sent within 30 units of time; if this happens, the task needs to resume the component by transitioning to the \emph{resuming} location, before sending the next message. In this problem, we include a safety property by requiring that the platform never transitions to the \emph{resuming} location, to avoid consuming excessive energy for the resumption process. Therefore, solution plans will be required to only travel between consecutive locations when the first message is sent, so that the communication component does not need to be resumed. We scale the instances by increasing the number of locations, by considering all the possible combinations of locations in which to send messages, and by increasing the bound $\kappa$.


$\textsc{Factory}$ is the same domain that we used in the running example shown in \cref{fig:factory_plans} and \cref{fig:factory_ta}. We consider two different ways of modeling the deadline: either at the planning layer (domain $\textsc{Factory1}$), by having a durative "Process" action that needs to be run in parallel with all other actions in the plan, or at the platform layer (domain $\textsc{Factory2}$), by having a component that synchronizes with the task associated with the "Work" action, and that disables the synchronization after the deadline has passed (we experiment with both versions of the problem). The instances are scaled by increasing the number of required Work actions, by considering different deadlines, and by increasing the bound $\kappa$.


We performed all the experiments on a cluster of identical machines with AMD EPYC 7413 24-Core Processor and running Ubuntu 20.04.6. We used a timeout of 14400 seconds and a memory limit of 20GB.
The experimental results are shown in \cref{fig:plots}. We can observe that both approaches are effective at solving the tested benchmarks, with the approach based on abstraction-refinement having a wider coverage and faster solving times. This is expected, especially if the number of necessay refinement loops is low, as heuristic-search based planners are typically much faster at finding solution plans compared to encoding-based approaches, and checking the executability and safety of a STN plan is computationally much cheaper compared to combining the check with the full encoding of the planning problem. In the tested benchmarks, the number of necessary loops in the second approach ranged from 1 to 8. It is evident from the coverage table that the bound $\kappa$ on the length of the platform traces greatly influences the performance of both approaches, as platform traces are universally quantified in the encoding formula. A future direction is to try to check the executability and safety notions using an unbounded technique, possibly proving the non-existence of bad traces.



\section{Conclusions}

In this paper, we formally defined the ``Platform-Aware Mission Planning'' problem, motivated by the need of synthesizing plans that not only achieve the mission objectives, but also ensure executability and the satisfaction of safety properties during execution. We devised an amalgamated method and a decomposition approach that can solve the problem, and showed the superiority of the latter experimentally.

As future work, we plan to generalize our model to the case of hybrid automata, allowing the representation of continuous behaviors and resources in the platform. Moreover, we are interested in other problems that can be defined in the formal framework we proposed, such as synthesizing plans that guarantee other formal properties like diagnosability.

\section*{Acknowledgments}

This work has been partly supported by the PNRR project iNEST -- Interconnected Nord-Est Innovation Ecosystem (ECS00000043) funded by the European Union NextGenerationEU program and by the STEP-RL project funded by the European Research Council (grant n. 101115870).

\bibliography{aaai25}

\end{document}